\documentclass[review, 10pt]{elsarticle}

\usepackage{amssymb}
\usepackage{amsmath}
\usepackage{siunitx}

\usepackage{multirow}
\usepackage{graphicx}
\usepackage{mathtools}
\usepackage{algorithm}
\usepackage{algorithmic}
\journal{Neurocomputing}

\begin{document}

\begin{frontmatter}



\title{Loss-controlling calibration with asymptotic guarantees}


\author[mymainaddress]{Di Wang}
\ead{wangdi2015@tju.edu.cn}
\author[mymainaddress]{Yujin Li}
\ead{liyujin@tju.edu.cn}
\author[mysecondaddress]{Junzhi Shi}
\ead{shijz@sdut.edu.cn}
\author[mythirdaddress]{Pingping Wang}
\ead{wangpingping@sdutcm.edu.cn}
\author[myforthaddress]{Shuo Zhuang}
\ead{shuozhuang@hfut.edu.cn}
\author[mymainaddress]{Hongyue Li\corref{mycorrespondingauthor}}
\ead{lihongyue@tju.edu.cn}
\cortext[mycorrespondingauthor]{Corresponding author}

\address[mymainaddress]{School of Electrical and Information Engineering, Tianjin University, Tianjin, 300072, Tianjin, China}
\address[mysecondaddress]{School of Electrical and Electronic Engineering, Shandong University of Technology, Zibo, 255000, Shandong, China}
\address[mythirdaddress]{Qingdao Academy of Chinese Medical Science, Shandong University of Traditional Chinese Medicine, Qingdao, 266112, Shandong, China}
\address[myforthaddress]{School of Computer Science and Information Engineering, Hefei University of Technology, Hefei, 230009, Anhui, China}


\begin{abstract}
We propose a learning framework for calibrating predictive models to perform loss-controlling prediction, which extends our proposed conformal loss-controlling prediction framework to more general settings. In contrast to existing methods, the predictors built by the proposed approach are not restricted to set-valued predictors, and the loss function can be any measurable function without requiring monotonicity assumptions. The approach is a natural extension of conformal loss-controlling prediction, since it reduces to the latter when the predictors are set-valued with the nesting property and the losses are monotone. The loss-controlling guarantee of the proposed framework is established in the asymptotic setting and is empirically verified in selective regression and high-impact weather forecasting, demonstrating its effectiveness for the domains requiring general loss-controlling prediction.
\end{abstract}




\end{frontmatter}



\section{Introduction}

Predictive models built on modern machine learning techniques have been deployed in many areas due to their expressive power. However, many of these algorithms cannot provide reliable information about the difference or distance between the prediction and the true label for a specific test object, which is essential for confidence prediction \cite{confidencedistributions} and is important for high-risk applications \cite{balasubramanian2014conformal}. When the prediction is a set of possible labels and the difference is the miscoverage loss of set predictors, the learning framework of conformal prediction (CP) can tackle this issue with its coverage guarantee under the assumption of exchangeable data samples \cite{vovk2005algorithmic} \cite{papadopoulos2026conformal}. 

The CP literature has expanded rapidly, including comprehensive surveys \cite{zhou2025dataperspective} and refinements that strengthen the classical marginal guarantee to conditional and group-valid forms of validity \cite{gibbs2025conditional}. Recently, CP has been adapted to data distributed across machines \cite{wang2026distributed}, to partially observed functional data handled jointly with registration \cite{wang2026functional}, to multi-horizon time series through joint and interpretable prediction regions \cite{english2026flipr}, and so on. Despite this breadth, all of these methods remain organized around set-valued prediction and the miscoverage loss and the guarantee they enforce is coverage of the label by a set.

To move beyond the miscoverage loss, risk-controlling prediction sets \cite{bates2021distribution} and conformal risk control (CRC) \cite{angelopoulos2022conformal} allow the prediction set to be calibrated against an arbitrary bounded loss that is monotone in the set size, and they guarantee that the expected loss of a future object does not exceed a user-specified risk level. This shift from a coverage event to a population-average risk has fueled a wave of applications. For example, expected false-discovery rates are controlled for selective question answering with foundation models \cite{wang2026coin}, for online candidate selection with accept-to-reject decisions \cite{liu2026ocsarc}, and for ranked retrieval through two-stage procedures \cite{xu2025twostage}; the expected frequency of unwanted recommendations is bounded in recommender systems \cite{detoni2026flowers}, and a coverage-calibrated deferral policy is integrated with physiological neural forecasting for glucose monitoring \cite{maqsood2026glucose}; the guarantee has also been sharpened from mean risk to distribution-free tail and spectral risks, such as conditional value-at-risk, for vision systems \cite{kazemi2026bqsrc}, and used to calibrate collective conservatism in the scalable oversight of autonomous agents \cite{overman2026oversight}. Crucially, however, every one of these guarantees bounds the risk averaged over the population from which the test object is drawn, rather than the loss of one specific test object.

To control the prediction loss for a specific test object, conformal loss-controlling prediction (CLCP) \cite{wang2023conformalloss} is proposed to extend CP from the miscoverage loss to losses satisfying monotonicity conditions, ensuring that, with high confidence, the prediction loss does not exceed a preset level. Both frameworks are nevertheless restricted to set-valued predictors and particular forms of losses, which motivates the present work: we consider general forms of predictors and losses for loss-controlling prediction, going beyond confidence sets with coverage guarantees.

The learning framework CLCP is inspired by risk-controlling prediction sets \cite{bates2021distribution} and conformal risk control \cite{angelopoulos2022conformal}, and aims to construct a set predictor $C_{\lambda^*}$ such that the prediction loss does not exceed the preset loss level $\alpha$ with probability no less than $1 - \delta$. To ensure the loss-controlling guarantee, the set predictors $C_{\lambda}$ should satisfy the nesting property for the parameter $\lambda \in \Lambda$ and the loss should be monotonic. Base on this, the optimal $\lambda^*$ is searched by calculating the $1 - \delta$ quantiles of the losses over the calibration data. Although CLCP extends CP to more general cases, the forms of the set predictors and loss functions that it admits remain limited.

To overcome this issue, one approach is to use the learn-then-test procedure \cite{angelopoulos2021learn} to combine multiple probability inequalities while retaining the controlling guarantee. However, this procedure cannot be applied effectively to our loss-controlling problem. One example is to use the Bonferroni correction to obtain the family-wise loss-controlling guarantee, where one needs to calculate the $1 - \delta/|\Lambda|$ quantile of losses for each possible $\lambda \in \Lambda$, resulting in meaningless calculation if $|\Lambda|$ is large and the number of calibration data is not. For example, if $|\Lambda| = 1000$ and $\delta = 0.1$, we need to calculate the $0.9999$ quantile of losses for each possible $\lambda$, which makes sense only if the number of calibration data is more than 10000.

Therefore, to improve data efficiency, the loss-controlling calibration (LCC) approach proposed in this paper employs predefined searching functions to avoid the multiple hypothesis testing process, and LCC can be applied to general predictive models and non-monotonic losses. Relaxing the assumptions employed in CLCP precludes a finite-sample loss-controlling guarantee, but we establish the guarantee asymptotically and verify the loss-controlling ability in selective regression and weather forecasting empirically. Across all experiments, the results confirm the effectiveness of the proposed LCC approach.

In summary, three contributions are made in this paper:
\begin{itemize}
	\item A learning framework, named loss-controlling calibration, is proposed for calibrating predictive models to perform general loss-controlling prediction. It is a natural extension of CLCP and is easy to implement.
	\item The loss-controlling guarantee of the proposed approach is established in the asymptotic setting, where both the predictor and the loss are more general than in existing loss-controlling prediction approaches.
	\item The proposed LCC is applied to selective regression and weather forecasting, empirically demonstrating its effectiveness for loss-controlling prediction in general settings.
\end{itemize}

The remaining parts of this paper are organized as follows. Section 2 reviews inductive conformal prediction and conformal loss-controlling prediction and Section 3 proposes the loss-controlling calibration approach with its theoretical analysis. Section 4 applies the proposed approach to selective regression and high-impact weather forecasting problems to empirically verify the loss-controlling guarantee. Finally, the conclusions of this paper are drawn in Section 5.

\section{Inductive conformal prediction and conformal loss-controlling prediction}

This section reviews inductive conformal prediction and recently proposed conformal loss-controlling prediction. Throughout this paper, let $\{(X_i, Y_i)\}_{i = 1}^{n}$ be $n$ calibration data and $(X_{0}, Y_{0})$ be the test object-response pair, which are all drawn independently from $P_{XY}$ on $\mathcal{X} \times \mathcal{Y}$. The lower-case letter  $(x_i, y_i)$ represents the realization of $(X_i, Y_i)$.

\subsection{Inductive conformal prediction}

Inductive conformal prediction (ICP) \cite{papadopoulos2008inductive} is a variant of conformal prediction tackling the computational issue of the original conformal prediction approach. ICP starts with any measurable function $A: \mathcal{X} \times \mathcal{Y} \rightarrow \mathcal{R}$ called nonconformity measure, and calculates $n$ nonconformity scores as
\begin{equation}\nonumber
	A_i = A(X_i, Y_i),
\end{equation}
for $i = 1, \cdots, n$.
Denote $Q^{(n)}_{1-\delta}$ as the $1 - \delta$ quantile of $\{A_i\}_{i=1}^n \cup \{\infty\}$. With the assumption of exchangeability of data samples, for any preset $\delta \in (0,1)$, ICP makes promise that
\begin{equation}\nonumber
	P \Bigg ( A(X_{0}, Y_{0}) \leq Q^{(n)}_{1-\delta} \Bigg ) \geq 1 - \delta.
\end{equation}
Thus, ICP outputs the following set prediction
\begin{equation}\nonumber
	C^{(n)}_{1-\delta}(X_{0}) = \{ y : A(X_{0}, y) \leq Q^{(n)}_{1-\delta}\},
\end{equation}
which leads to
\begin{equation}\nonumber
	P \Bigg (Y_{0} \in C^{(n)}_{1-\delta}(X_{0}) \Bigg ) \geq 1 - \delta.
\end{equation}
The nonconformity measure $A$ is usually designed based on a point prediction model $f$ trained on  training samples drawn from $P_{XY}$ and here is an example for a classification problem with $K$ classes. In this situation, based on training samples, one can train a classifier $f: \mathcal{X} \rightarrow [0,1]^K$, whose $k$th output $f_k$ is the estimated probability of the $k$th class, and the corresponding nonconformity measure can be defined as 
\begin{equation}\nonumber
	A(x, y) = 1 -f_k(x),
\end{equation}
which leads to the following prediction set for an input object $x$,
\begin{equation}\nonumber
	C^{(n)}_{1-\delta}(x) = \{ k : f_k(x) \geq 1 - Q^{(n)}_{1-\delta}\}.
\end{equation}

\subsection{Conformal loss-controlling prediction}

Different from ICP, the purpose of CLCP is to build predictors with loss-controlling guarantees, whose approach is inspired by conformal risk control \cite{angelopoulos2022conformal}. CLCP starts with a set-valued function $C_{\lambda}: \mathcal{X} \rightarrow \mathcal{Y}'$ with a parameter $\lambda \in \Lambda$, where $\Lambda$ is a discrete set of possible real values of $\lambda$ such as from $0$ to $1$ with step size of $0.01$. $\mathcal{Y}'$ denotes some space of sets. For example, $\mathcal{Y}'$ can be the power set of $\mathcal{Y}$ for 
single-label classification and can be equal to $\mathcal{Y}$ for binary image segmentation. This set-valued function $C_{\lambda}$ needs to satisfy the following nesting property introduced in \cite{bates2021distribution}:
\begin{equation}
	\lambda_1 < \lambda_2 \ \Longrightarrow\  C_{\lambda_1}(x) \subseteq C_{\lambda_2}(x). \label{formula_3}
\end{equation}
Here we give an example of constructing the prediction set $C_{\lambda}$ for classification problem with $K$ classes. With the same meanings of $f$ and $f_k$ mentioned in Section 2.1, CLCP can construct the prediction set as
\begin{equation}\nonumber
	C_{\lambda}(x) = \{ k : f_k(x) \geq 1 - \lambda\},
\end{equation}
which satisfies the nesting property of formula (\ref{formula_3}).

In addition, for each realization of response $y$, the loss function $L: \mathcal{Y} \times \mathcal{Y}' \rightarrow \mathcal{R}$ considered in CLCP should respect the following monotone property or nesting property:
\begin{equation}
	S_1 \subseteq S_2 \subseteq \mathcal{Y}' \ \Longrightarrow \ L(y, S_2) \leq L(y, S_1) \leq B, \label{formula_4}
\end{equation}
where $B$ is the upper bound.

After determining $C_{\lambda}$ and $L$, for preset $\alpha$ and $\delta$, CLCP \textbf{first} calculates $L_i$ as 
\begin{equation}\nonumber
	L_i(\lambda) = L(Y_i, C_{\lambda}(X_i))
\end{equation}
for $i = 1, \cdots, n$, and \textbf{then} searches for $\lambda^*$ such that
\begin{equation}\nonumber
	\lambda^* = \min \Bigg \{ \lambda \in \Lambda: Q^{(n)}_{1-\delta}(\lambda) \leq \alpha    \Bigg \},
\end{equation}
where $Q^{(n)}_{1-\delta}(\lambda)$ is the $1 - \delta$ quantile of $\{L_i(\lambda)\}_{i=1}^n \cup \{ B \}$. The finally obtained set predictor $C_{\lambda^*}$ satisfies the loss-controlling guarantee, which is proved in theory for distribution-free and finite-sample conditions, and CLCP can be seen as an extension of CP for specific forms of $C_{\lambda}$ and $L$ \cite{wang2023conformalloss}.

\section{Loss-controlling calibration and its theoretical analysis}

This section introduces the extension of CLCP to general cases with the proposed loss-controlling calibration and proves its loss-controlling guarantee in the asymptotic setting.

\subsection{Loss-Controlling Calibration}
CLCP needs nesting properties for $C_{\lambda}$ and $L$, which limits its applicability. Therefore, we propose loss-controlling calibration for general predictors and loss functions. We denote $F_{\lambda}$ as a predictor built on a predictive model $f$ learned from training data, where $\lambda$ is a parameter taking values from a discrete set $\Lambda$. For LCC, we emphasize that $F_{\lambda}: \mathcal{X} \rightarrow \mathcal{Y}'$ can be any kind of predictor, i.e., $\mathcal{Y}'$ does not have to be the set of label sets. Besides, $\Lambda$ can be any discrete set such as the set of multi-dimensional vectors as in \cite{angelopoulos2021learn}. Also, the loss function $L: \mathcal{Y} \times \mathcal{Y}' \rightarrow \mathcal{R}$ considered for LCC can be any measurable function bounded above by $B$, i.e., for each object-response pair $(x, y)$,
\begin{equation}
	L(y, F_{\lambda}(x)) \leq B. \label{formula_5}
\end{equation}

Given these general conditions, one way of constructing loss-controlling guarantee is to use multiple hypothesis testing process developed in learn then test \cite{angelopoulos2021learn}. However, this may lead to calculating the $1 - \delta/|\Lambda|$ quantiles of losses on calibration data, which may be meaningless for our loss-controlling approach when the number of calibration data is small or moderate. Thus, we propose to use a predefined function $s$ independent of $\{(X_i, Y_i)\}_{i = 1}^{n+1}$ to do the trick, where $s$ stands for searching since it can be defined as an optimization algorithm for parameter searching. The approach of LCC is very similar to CLCP and we first introduce it for comparison, leaving the analysis of it to the next section. 

After determining $F_{\lambda}$ and $L$, for preset $\alpha$ and $\delta$, LCC \textbf{first} calculates $L_i$ on calibration data as
\begin{equation}
	L_i(\lambda) = L(Y_i, F_{\lambda}(X_i)), \label{formula_6}
\end{equation}
and \textbf{then} search for $\lambda^*_n$ such that 
\begin{equation}
	\lambda^*_n = s \Bigg (\Bigg \{ \lambda \in \Lambda: Q^{(n)}_{1-\delta}(\lambda) \leq \alpha    \Bigg \} \Bigg ), \label{formula_7}
\end{equation}
where $s: \mathcal{P}(\Lambda) \rightarrow \Lambda$ is the predefined searching function defined on the power set of $\Lambda$, whose output is an element of its input, and $Q^{(n)}_{1-\delta}(\lambda)$ is the $1-\delta$ quantile of $\{L_i(\lambda)\}_{i=1}^n \cup \{ B \}$. The final predictor built by LCC is $F_{\lambda^*_n}$.

It can be seen that LCC is exactly CLCP if $\Lambda \subset \mathcal{R}$, $F_{\lambda}$ is a set predictor with nesting property as formula (\ref{formula_3}), $L$ is monotone as formula (\ref{formula_4}) and $s$ is the min function. Therefore, for LCC we also use the same notations of $L_i$ and $Q^{(n)}_{1-\delta}(\lambda)$ as CLCP to represent similar concepts. Here we summarized LCC in Algorithm 1.

\begin{algorithm}
	\caption{Loss-Controlling Calibration}
	\label{alg:Framwork}
	\begin{algorithmic}[1]
		\REQUIRE ~~\\
		Calibration dataset $\{(x_i, y_i)\}_{i=1}^n$, test input object $x_{0}$, the predictor $F_{\lambda}$, the loss function $L$ satisfying formula (\ref{formula_5}), the predefined searching function $s$, preset $\alpha \in \mathcal{R}$ and $\delta \in (0, 1)$.\\
		\ENSURE ~~\\
		Calibrated prediction for $y_{0}$.
		\STATE
		Based on formula (\ref{formula_6}), calculate $\{L_i(\lambda)\}_{i=1}^n$.
		\STATE
		Search for $\lambda^*_n$ satisfying formula (\ref{formula_7}).
		\RETURN
		$F_{\lambda^*_n}(x_{0})$
	\end{algorithmic}
\end{algorithm}

\subsection{Theoretical analysis of loss-controlling calibration}

We first introduce some notations and then provide the asymptotic analysis for Algorithm 1. For fixed $\lambda$, let $R(\lambda)$ be the expectation of the event that the loss is above the loss level,
\begin{equation}\nonumber
	R(\lambda) = E \Big [ \mathbb{I}\Big \{ L \Big (Y_0, F_{\lambda}(X_0) \Big ) > \alpha \Big \} \Big ],
\end{equation}
and let $\hat R(\lambda)$ be the empirical version,
\begin{equation}\nonumber
\hat R(\lambda) = \frac{1}{n}\sum_{i=1}^n \mathbb{I}\Big \{ L \Big (Y_{i}, F_{\lambda}(X_{i}) \Big ) > \alpha \Big \},
\end{equation}
where $\mathbb{I}$ is the indicator function. The loss-controlling guarantee we prove is that for any $\epsilon > 0$, there holds
\begin{equation}\nonumber
\lim_{n \rightarrow \infty}P \Bigg ( R(\lambda^*_n) \leq \delta + \epsilon
\Bigg ) = 1,
\end{equation}
which is equivalent to 
\begin{equation}\nonumber
	\lim_{n \rightarrow \infty}P \Bigg ( R(\lambda^*_n) > \delta + \epsilon
	\Bigg ) = 0.
\end{equation}
This means that the amount by which $R(\lambda^*_n)$ exceeds $\delta$ converges to $0$ in probability, i.e., 
\begin{equation}\nonumber
	(R(\lambda^*_n) - \delta )_{+} \xlongrightarrow{\mathbb{P}} 0,
\end{equation}
where \((x)_+ = \max\{x,0\}\). 

\newtheorem{theorem}{Theorem}
\begin{theorem}
	Suppose the calibration dataset $\{(X_i, Y_i)\}_{i = 1}^{n}$ and the test sample $(X_0, Y_0)$ are all drawn independently from $P_{XY}$ on $\mathcal{X} \times \mathcal{Y}$, $F_{\lambda}: \mathcal{X} \rightarrow \mathcal{Y}'$ is a function with the parameter $\lambda$ taking values from a discrete finite set $ \Lambda$, $L: \mathcal{Y} \times \mathcal{Y}' \rightarrow \mathcal{R}$ is a loss function satisfying formula (\ref{formula_5}) and $L_i(\lambda)$ is defined as formula (\ref{formula_6}). Denote $s: \mathcal{P}(\Lambda) \rightarrow \Lambda$ as any searching function defined on the power set of $\Lambda$ whose output is the element of its input. For any preset $\alpha \in \mathcal{R}$ and preset $\delta \in (0,1)$, if $L$ also satisfies the following conditions,
	\begin{equation}
		\min_{\lambda}\max_i L_i(\lambda) < \alpha, \label{formula_9}
	\end{equation}
	then for any $\epsilon > 0$, there holds
	\begin{equation}\nonumber
	\lim_{n \rightarrow \infty}P \Bigg ( R(\lambda^*_n) \leq \delta + \epsilon
	\Bigg ) = 1.
	\end{equation}
\end{theorem}

\newproof{proof}{Proof}
\begin{proof}
The condition $\delta \in (1/(1 + n), 1)$ is a sufficient condition ensuring the existence of $\lambda^*$ \cite{wang2023conformalloss}. Since we aim to prove the limit expression, we only need to consider sufficiently large $n$. For the preset $\delta \in (0,1)$, all $n > \lceil 1/\delta \rceil$ implies the existence of $\lambda^*$. Therefore, in the following proof, we assume that $n > \lceil 1/\delta \rceil$.

Let $\hat Q^{(n)}_{1-\delta}(\lambda)$ be the quantile of the empirical losses $\{L_i(\lambda)\}_{i=1}^n$. Therefore, we have
\begin{equation}\nonumber
\hat Q^{(n)}_{1-\delta}(\lambda^*_n) \leq Q^{(n)}_{1-\delta}(\lambda^*_n) \leq \alpha
\end{equation}
Then by definition of the quantile, there holds
\begin{equation}\nonumber
	\hat Q^{(n)}_{1-\delta}(\lambda^*_n) \leq \alpha \iff \hat R(\lambda^*_n) \leq \delta.
\end{equation}

For fixed $\epsilon > 0$, define the random set $\mathcal{B}_{\epsilon}$ as follows,
\begin{equation}\nonumber	
	\mathcal{B}_{\epsilon} = \left\{\lambda \in \Lambda : R(\lambda) > \delta + \epsilon, \hat{R}(\lambda) \leq \delta \right\}.
\end{equation}
The inequality $\hat R(\lambda^*_n) \leq \delta$ implies that
\begin{equation}\nonumber	
	\left\{R(\lambda^*_n) > \delta + \epsilon\right\} \subseteq \left\{\mathcal{B}_{\epsilon} \neq \varnothing \right\},
\end{equation}
which leads to
\begin{equation}\nonumber	
	P \Big (R(\lambda^*_n) > \delta + \epsilon \Big ) \leq P \Big (\mathcal{B}_{\epsilon} \neq \varnothing \Big ).
\end{equation}
Since
\begin{equation}\nonumber
	\left\{\mathcal{B}_{\epsilon} \neq \varnothing\right\} = \bigcup_{\lambda \in \Lambda} \left\{R(\lambda) > \delta + \epsilon,\ \hat{R}(\lambda) \leq \delta \right\},
\end{equation}
we have
\begin{equation}\nonumber	
	P \Big (\mathcal{B}_{\epsilon} \neq \varnothing \Big ) \leq \sum_{\lambda \in \Lambda} P \Big (R(\lambda) > \delta + \epsilon,\ \hat{R}(\lambda) \leq \delta \Big ).
\end{equation}
For fixed $\lambda \in \Lambda$, if $R(\theta) \leq \delta + \epsilon$, we have 
\begin{equation}\nonumber	
P \Big (R(\lambda) > \delta + \epsilon,\ \hat{R}(\lambda) \leq \delta \Big ) = 0.
\end{equation}
If $R(\lambda) > \delta + \epsilon$, we have 
\begin{equation}\nonumber	
	P \Big (R(\lambda) > \delta + \epsilon,\ \hat{R}(\lambda) \leq \delta \Big ) \leq P \Big (\hat{R}(\lambda) - R(\lambda) \leq -\epsilon \Big ) \leq \exp \Big (-2n\epsilon^2 \Big ),
\end{equation}
by Hoeffding's inequality. Thus, there holds
\begin{equation}\nonumber
	\mathbb{P}\!\left(\mathcal{B}_{\epsilon} \neq \varnothing\right) \leq |\Lambda|\exp\!\left(-2n\epsilon^2\right),
\end{equation}
which leads to 
\begin{equation}\nonumber	
	P \Big (R(\lambda^*_n) > \delta + \epsilon \Big ) \leq |\Lambda|\exp\!\left(-2n\epsilon^2\right).
\end{equation}
Therefore, 
\begin{equation}\nonumber
	P \Big (R(\lambda^*_n) \leq \delta + \epsilon \Big ) \geq 1 -  |\Lambda|\exp\!\left(-2n\epsilon^2\right),
\end{equation}
and the Theorem 1 is proved by taking $n \rightarrow \infty$.
\end{proof}

The proof of Theorem 1 follows the recent theoretical work about conformal risk control considering non-monotonic losses \cite{angelopoulos2026conformal}, and the core step is the construction $\mathcal{B}_{\epsilon}$ and the probability bound given by Hoeffding's inequality. The loss-controlling guarantee is empirically verified in the next section.

\section{Experiments}

In this section, we first apply LCC to selective regression on $20$ public datasets for single-target regression, which tests the ability of LCC to calibrate point predictors. Then we introduce LCC to post-processing of weather forecasting \cite{trotta2025aies} \cite{mengist2026mwr} \cite{tong2026grl}, and we focus on high-impact weather forecasting applications to control the non-monotone loss related to false discovery. We use Python \cite{citepython} to conduct the experiments. The statistical learning methods in Section 4.1 were coded with Scikit-learn \cite{scikit-learn} and the deep neural nets in Section 4.2 were coded with Pytorch \cite{NEURIPS2019_9015}. 

\begin{table}[h]
	\centering
	\caption{Datasets for Single-Target Regression}
	\scalebox{0.9}{
		\begin{tabular}{lccc}
			\hline
			Dataset & Examples & Dimensionality &  Source \\
			\hline
			abalone & 4177 & 8 & UCI \\
			bank8fh & 8192& 8 & Delve \\
			bank8fm & 8192 & 8 & Delve \\
			bank8nh & 8192 & 8 & Delve \\
			bank8nm & 8192 & 8 & Delve \\
			boston & 506 & 13 & UCI \\
			cooling & 768 & 8 & UCI \\
			heating & 768 & 8 & UCI \\
			istanbul & 536 & 7 & UCI \\
			kin8fh & 8192 & 8 & Delve \\
			kin8fm & 8192 & 8 & Delve \\
			kin8nh & 8192 & 8 & Delve \\
			kin8nm & 8192 & 8 & Delve \\
			laser & 993 & 4 &  KEEL\\
			puma8fh & 8192 & 8 &  Delve\\
			puma8fm & 8192& 8&  Delve\\
			puma8nh & 8192 & 8 &  Delve\\
			puma8nm & 8192 & 8 &  Delve\\
			stock & 950 & 9 &  KEEL\\
			treasury & 1048 & 15 &  Delve\\
			\hline
	\end{tabular}}
\end{table}

\subsection{LCC for selective regression with single target}

\begin{figure*}[h]
	\centering
	\includegraphics[width = 1 \hsize]{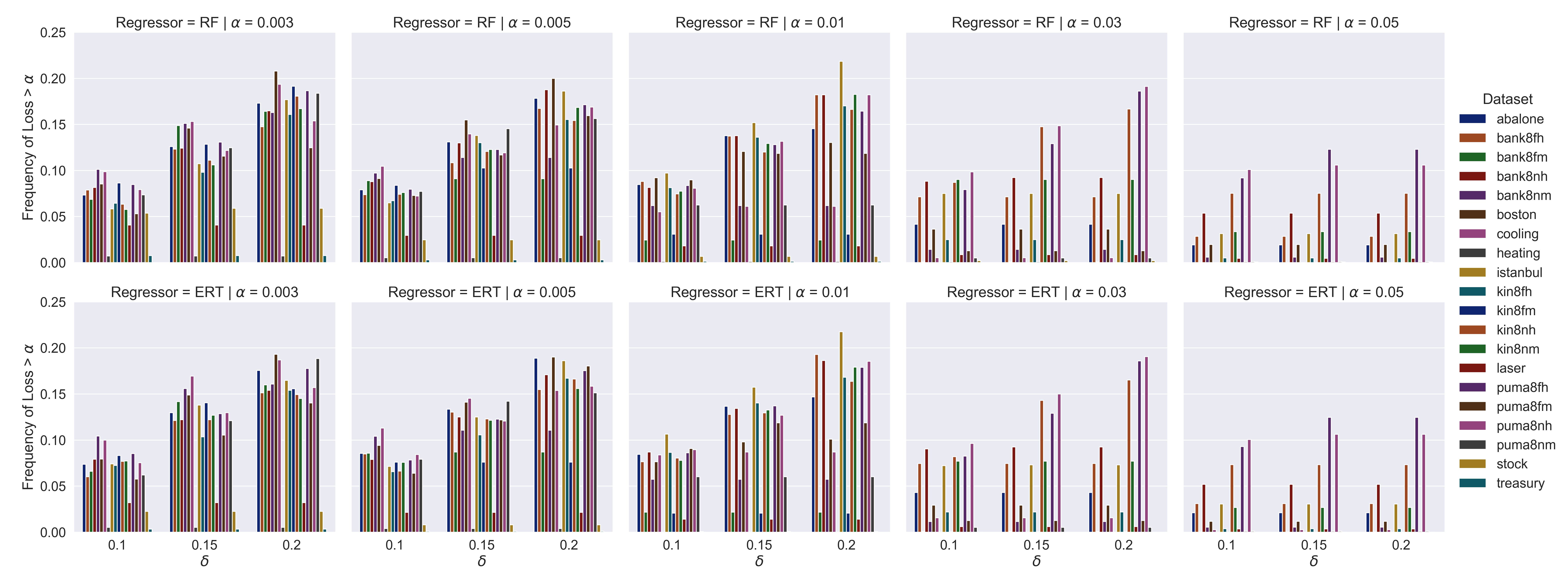}
	\caption{Frequencies of the prediction losses being greater than $\alpha$ vs. $\delta = 0.1, 0.15, 0.2$ on test data for selective single-target regression. The first row and the second row correspond to RF and ERT respectively. Different columns represent different $\alpha$. The bars are all near or below $\delta$, indicating the controlling guarantee of LCC empirically.}
\end{figure*}

\begin{figure*}[h]
	\centering
	\includegraphics[width = 1 \hsize]{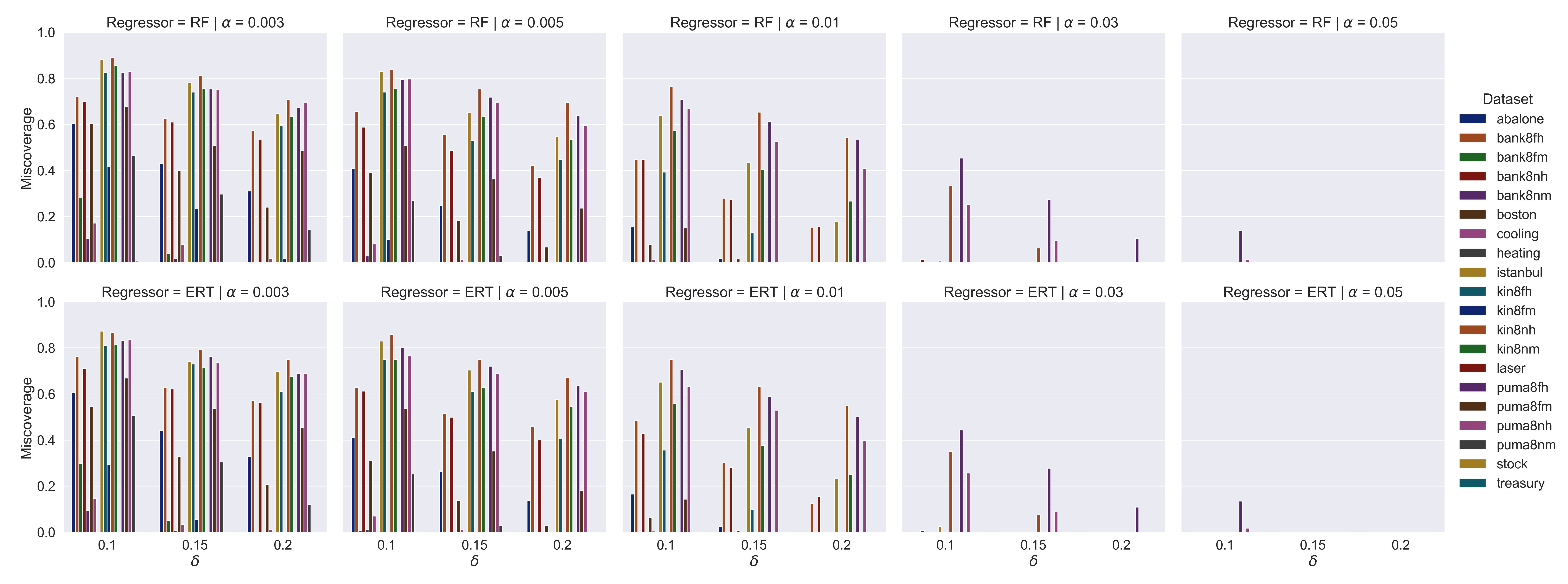}
	\caption{Miscoverage of selective predictions vs. $\delta = 0.1, 0.15, 0.2$ on test data for selective single-target regression. The first row and the second row correspond to RF and ERT respectively. Different columns represent different $\alpha$. Tuning $\alpha$ and $\delta$ can change Miscoverage, which indicates the trade-off between loss level and informational efficiency.}
\end{figure*}

\begin{figure*}[h]
	\centering
	\includegraphics[width = 1 \hsize]{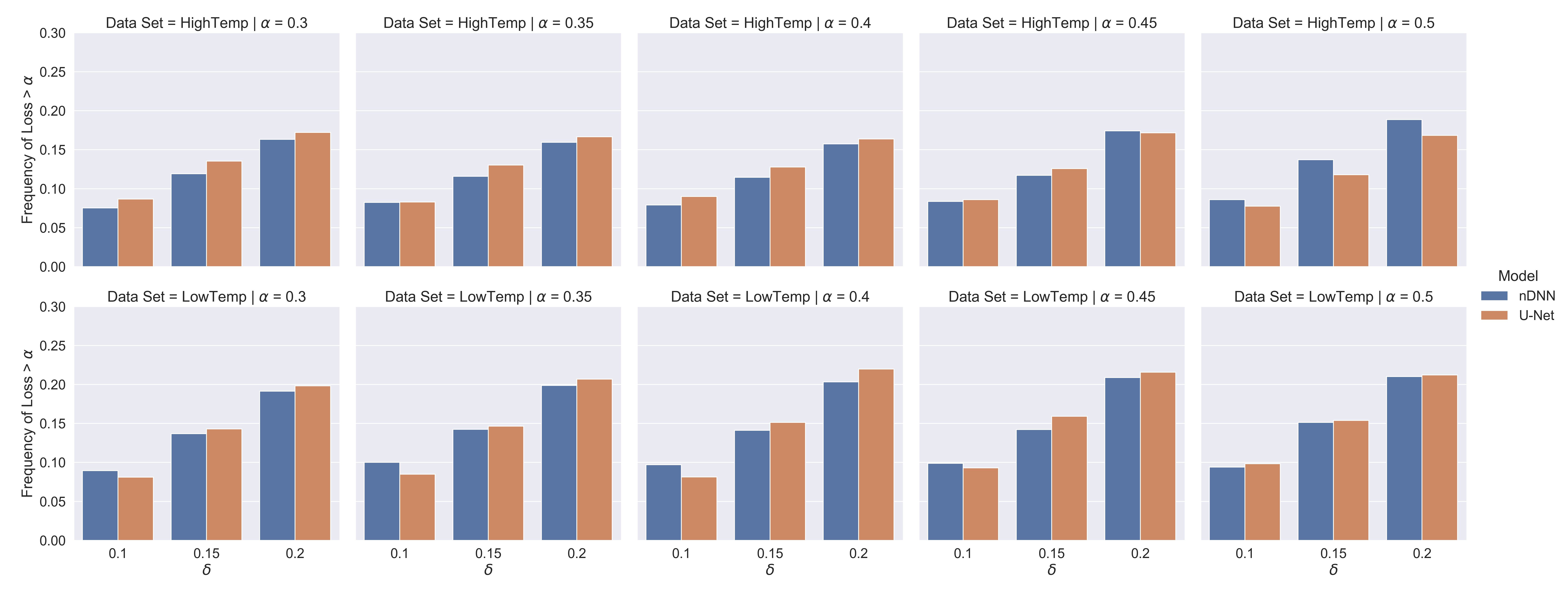}
	\caption{Frequencies of the prediction losses being greater than $\alpha$ for different $\delta$ and $\alpha$ on test data of HighTemp and LowTemp datasets. All bars being near or below the preset $\delta$ confirms the controlling guarantee of LCC empirically.}
\end{figure*}

\begin{figure*}[h]
	\centering
	\includegraphics[width = 1 \hsize]{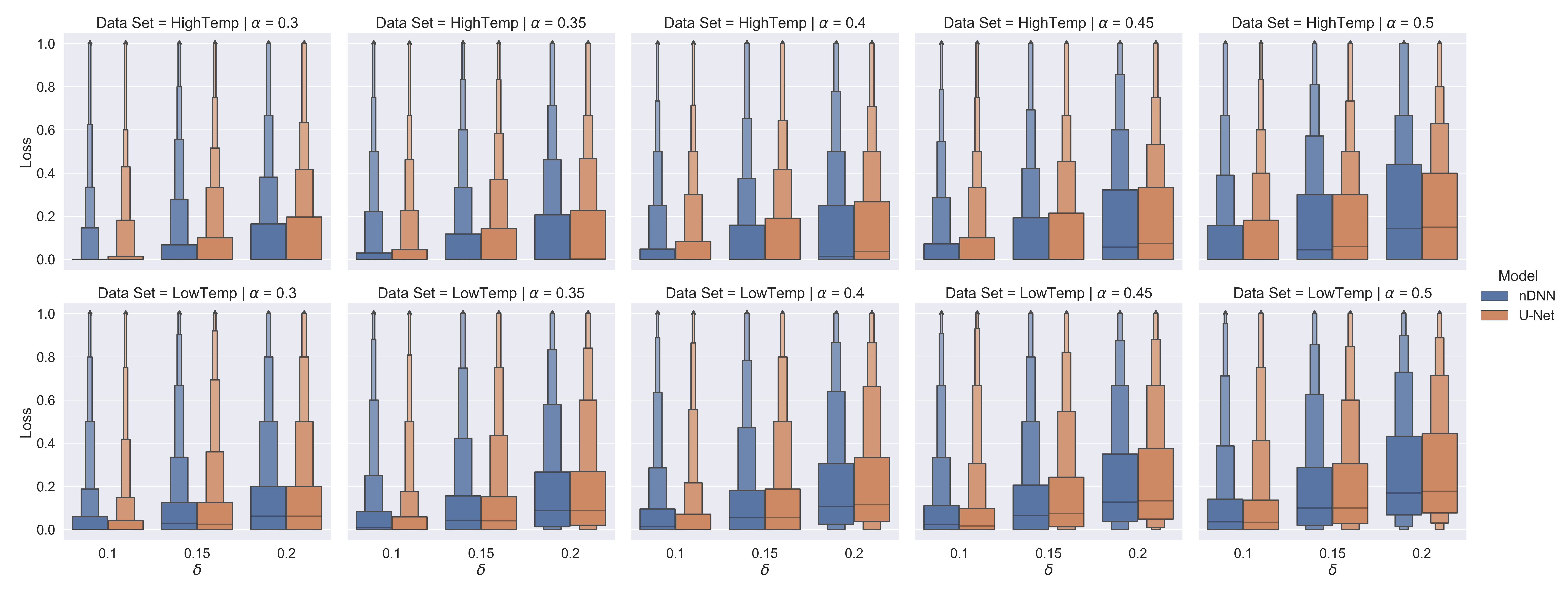}
	\caption{Distributions of the prediction losses for different $\delta$ and $\alpha$ on test data of HighTemp and LowTemp datasets. The losses are controlled by $\alpha$ and $\delta$ properly to achieve the empirical validity in Figure 3.}
\end{figure*}

\begin{figure*}[h]
	\centering
	\includegraphics[width = 1 \hsize]{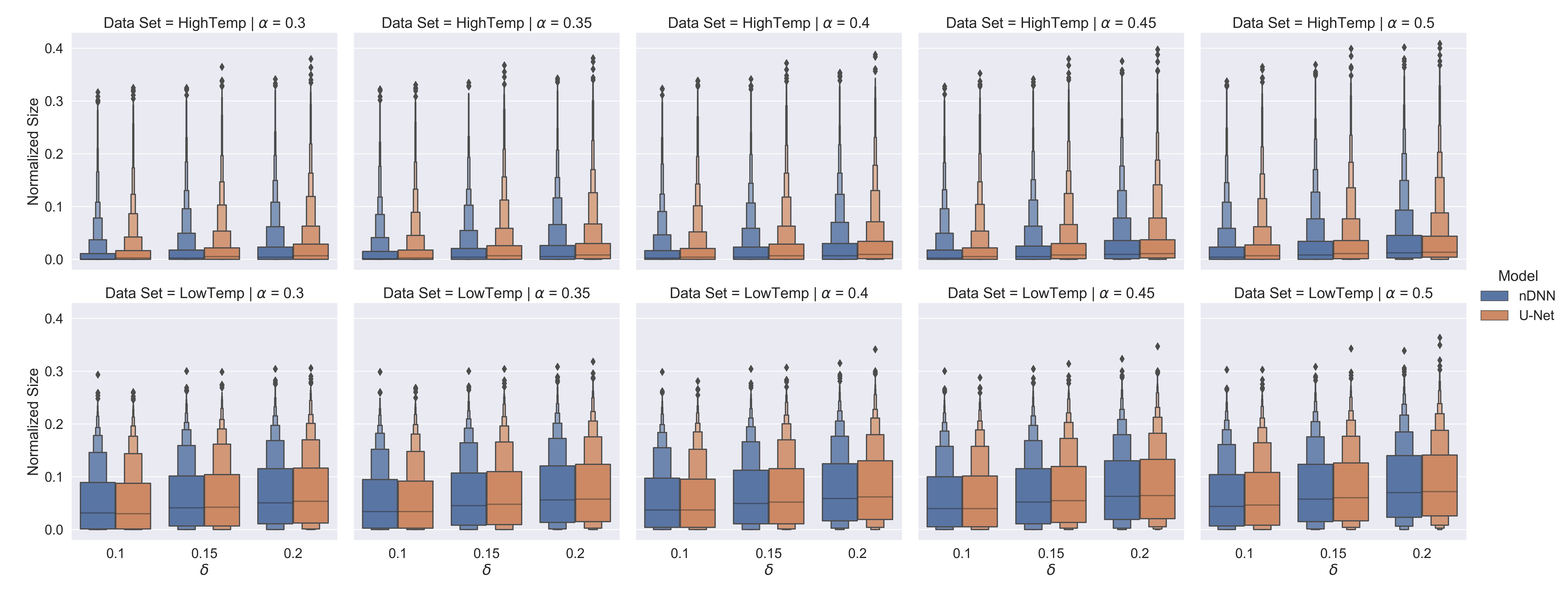}
	\caption{Distributions of normalized sizes for different $\delta$ and $\alpha$ on test data of HighTemp and LowTemp datasets. The predictions have reasonable sizes for both U-Net and nDNN for high-impact weather forecasting.}
\end{figure*}

Selective regression is the selective prediction task for regression problems, which we introduce here referring to \cite{Feng2022selective}. Selective regression model is a model with the ability to abstain from making prediction when lacking confidence. The model can be comprised of a prediction function $f$ and a hard selection function $g$. For a given input object $x$, a selective model predicts the label as $f(x)$ if $g(x) = 1$, and abstains from making prediction if $g(x) = 0$. The hard selection function $g$ can be built on a soft selection function $\bar{g}$ and a threshold $\lambda$, such that $g(x) = 1$ if $\bar{g}(x) \leq \lambda$ and  $g(x) = 0$ if $\bar{g}(x) > \lambda$. An example is to use a function related to estimated conditional variance as $\bar{g}(x)$.

The objective of a selective regression model can be formalized as minimizing the risk $E((Y - f(X))^2g(X))$ in the condition that $E(g(X))$ is not too low, where $E(g(X))$ is a performance indicator called coverage. Therefore, we conduct experiments in this section to build selective regression models to control the prediction loss $(y - f(x))^2g(x)$, whose informational efficiency is measured by $1 - E(g(X))$, which is estimated using test data and is recorded as miscoverage in Figure 2. Lower miscoverage means better performance for selective models.

The empirical studies were conducted on $20$ public datasets for single-target regression, which are from Delve \cite{rasmussen1996delve}, KEEL \cite{alcala2010keel} and UCI \cite{asuncion2007uci} repositories and the information is summarized in TABLE I.

We employ bagging trees to build selective regression models, and the corresponding prediction function $f(x)$ and soft selection function $\bar{g}(x)$ are constructed using the mean and the standard variance of the predictions made by tree members. The selective regression model for calibration in this paper is formalized as 
\begin{equation}\nonumber
	F_{\lambda}(x)=\begin{cases}
		f(x)&\bar{g}(x) \leq \lambda,\\
		\emptyset&\bar{g}(x) > \lambda.
	\end{cases}
\end{equation}
All features and labels were normalized to $[0,1]$ with min-max normalization. For each dataset, $20\%$ of the data were used for testing and $80\%$ and $20\%$ of the remaining data were used for training and calibration. Random forests (RF) \cite{breiman2001random} and extremely randomized trees (ERT) \cite{geurts2006extremely} were used to build selective regression models with the default meta-parameters set by Scikit-learn. The data split process for each dataset was randomly conducted $10$ times and the average results were recorded in Figure 1 and Figure 2, where we set $\alpha \in \{0.003, 0.005, 0.01, 0.03, 0.05\}$ and $\delta \in \{0.1, 0.15, 0.2\}$. The $\lambda^*_n$ is searched from $0$ to $1$ with step size being $0.01$, and the search function $s$ is the max function, since we prefer low miscoverage for selective regression models.

The bar plots in Figure 1 demonstrate that the frequency of the losses being above $\alpha$ is near or below preset $\delta$, which verifies the loss-controlling guarantee of LCC for point predictors empirically, since the frequency is an estimation of the following probability 
\begin{equation}\nonumber
	P \Bigg (L \Big (Y_{0}, F_{\lambda^*_n}(X_{0}) \Big ) > \alpha \Bigg ),
\end{equation}
which we expect to be below $\delta$.
In Figure 2, we can observe that tuning $\alpha$ and $\delta$ can change miscoverage of selective regression models, which is reasonable since high levels of loss and significance relax the constrains on the prediction losses. This indicates that one should set $\alpha$ and $\delta$ properly for specific applications, making the trade-off between prediction loss and miscoverage.

\subsection{LCC for high-impact weather forecasting}
We apply LCC to high-impact weather forecasting, which is based on postprocessing of numerical weather prediction (NWP) models \cite{vannitsem2018statistical} \cite{vannitsem2021statistical} \cite{gronquist2021deep}, i.e., learning a predictor whose inputs are forecasts made by NWP models and outputs are corresponding high-impact weather. We use LCC to postprocess the ensemble forecasts issued by European Centre for Medium-Range Weather Forecasts (ECMWF) \cite{li2026investigating}. The forecasts are obtained from the THORPEX Interactive Grand Global Ensemble (TIGGE) dataset \cite{cisl_rda_ds330}. We concentrate on $2$-m maximum temperature and minimum temperature forecasts initialized at $0000$ UTC with the forecast lead times from $12$nd hour to $36$th hour. The resolution of the forecast fields is $0.5^{\circ} \times 0.5^{\circ}$ and the corresponding label fields with the same resolution are calculated using the ERA5 reanalysis data \cite{hersbach2020era5}. The area covers the main parts of North China, East China and Central China, ranging from $109^{\circ}$E to $122^{\circ}$E in longitude and from $29^{\circ}$N to $42^{\circ}$N in latitude with the grid size being $27 \times 27$. The HighTemp and LowTemp datasets introduced in \cite{wang2023conformalloss} are used for empirical studies. The inputs in HighTemp are $2$-m maximum temperature forecasting fields and the corresponding label fields are whether the observed $2$-m maximum temperature is above $\SI{35}{\degreeCelsius}$ for each grid. Similarly, the inputs in LowTemp are $2$-m minimum temperature forecasting fields and the corresponding label fields are whether the observed $2$-m minimum temperature is below $\SI{-15}{\degreeCelsius}$ for each grid. The sample sizes of HighTemp and LowTemp are $1200$ and $1233$ respectively.

The experimental setting is similar to that in Section IV-B of \cite{wang2023conformalloss}. For each dataset, all forecasts made by the NWP model were normalized to $[0,1]$ by min–max normalization. $20\%$ of the data were used for testing and $80\%$ and $20\%$ of the remaining data were used for training and calibration respectively. The normalized ensemble fields forecast by the NWP model are taken as input $x$ and the set of grids having high-impact weather is the corresponding label $y$, which can be seen as the image segmentation problem in computer vision. Thus, we employed two fully convolutional neural networks \cite{reza2026comprehensive} as our underlying algorithms. One was U-Net \cite{asadpour2026artificial} and the other is the naive deep neural network (nDNN), which is the U-Net removing skip-connections. The structures of the two networks are the same as those in \cite{wang2023conformalloss}. To train the deep nets, we further partitioned the data for training to validation part ($10\%$) and proper training part ($90\%$), which were used for model selection and parameter updating respectively. Adam optimization \cite{ibrahim2026improving} with the learning rate being $0.0001$ and the number of epochs being $1000$ was employed for training, and the model whose binary cross entropy was the lowest on validation data was chosen as the predictive model $f$ needing calibration. The candidate calibrated predictor $F_{\lambda}$ is defined as
\begin{equation}\nonumber
	F_{\lambda}(x) = \{(p,q): f_{(p,q)}(x) \geq \lambda\},
\end{equation}
where $f_{(p,q)}(x)$ is the estimated probability for high-impact weather existing at grid $(p,q)$.
The loss function is
\begin{equation}\nonumber
	L(y, F) = 1 - \frac{|y \cap F|}{|F|},
\end{equation}
which is a non-monotone loss function related to false discovery introduced in \cite{angelopoulos2021learn}, and can be seen as one minus precision for each sample. The searching function $s$ we used is the min function, as we expect to detect more high-impact weather given the precision for each sample being controlled properly. We also tested other forms of searching functions such as the max function. However, although the controlling guarantee can be hold empirically, the constructed predictor may lose informational efficiency for applicability, implying that the forms of searching functions should be designed on a case-by-case basis. The final calibrated predictors were obtained with the proposed LCC approach and the experimental results are shown in Figure 3, Figure 4 and Figure 5.

The frequencies of the prediction losses being more than $\alpha$ are shown in Figure 3 with bar plots for $\delta = 0.1, 0.15$ and $0.2$. The columns represent the cases where $\alpha = 0.3, 0.35, 0.4, 0.45$ and $0.5$ respectively. All bars are near or below the preset $\delta$, which verifies loss-controlling guarantee empirically. The boxen plots of the losses for different $\delta$ and $\alpha$ are shown in Figure 4, which contain more information about tails by drawing narrower boxes than box plots. It can be observed that $\alpha$ and $\delta$ result in larger losses, which should be preset based on specific applications. The informational efficiency of $F_{\lambda^*_n}$ is measured using normalized size of the prediction set defined as $|F_{\lambda^*_n}(x)| /PQ$ in which the numbers of the vertical and the horizontal grids of prediction fields are denoted by $P$ and $Q$ respectively.
The distributions of normalized sizes are shown in Figure 5, indicating that different $\alpha$ and $\delta$ cause different normalized sizes and there should be a trade-off among loss level $\alpha$, confidence level $1-\delta$ and informational efficiency of the predictions. Finally, all of the predictions have reasonable sizes using LCC, which demonstrates its effectiveness for high-impact weather forecasting.

\section{Conclusion}

This paper proposes loss-controlling calibration, which extends conformal loss-controlling prediction to calibrating predictive models with more general forms of predictors and losses, whose asymptotic loss-controlling guarantee is proved theoretically. The proposed loss-controlling calibration is natural extension of conformal loss-controlling prediction as the main difference between them is whether the calibrated predictors and the loss functions satisfy specific conditions. The method is applied to selective regression and high-impact weather forecasting problems, and the loss-controlling guarantee is verified empirically in these cases, showing the promising performance for controlling general losses. Given its generality, extending loss-controlling calibration to a broader range of research domains will be the direction of our future work.

\section*{Acknowledgement}

The authors would like to thank the anonymous editor and reviewers for their valuable comments and suggestions which improved this work.

This work was supported in part by the National Natural Science Foundation of China under Grant 62106169 and Grant 62402297, and in part by the Natural Science Foundation of Shandong Province under Grant ZR2023QF094.












\bibliographystyle{elsarticle-num} 
\bibliography{ref}

\end{document}